\documentclass{article}

\usepackage[nonatbib, preprint]{neurips_2023}

\usepackage[utf8]{inputenc} 
\usepackage[T1]{fontenc}    
\usepackage{hyperref}       
\usepackage{url}            
\usepackage{booktabs}       
\usepackage{amsfonts}       
\usepackage{nicefrac}       
\usepackage{microtype}      
\usepackage{xcolor}         
\usepackage{wrapfig}
\usepackage{tikz}
\usetikzlibrary{shapes,arrows,chains}
\usetikzlibrary[calc]
\tikzstyle{line} = [draw, -latex']

\usepackage{biblatex}
\usepackage{multirow, makecell}
\usepackage{algorithm}
\usepackage{algorithmic}

\usepackage{graphicx}
\usepackage{subfig}
\usepackage{amsmath}
\usepackage{amssymb}
\usepackage{mathtools}
\usepackage{amsthm}

\theoremstyle{plain}
\newtheorem{theorem}{Theorem}[section]

\newtheorem{lemma}[theorem]{Lemma}
\newtheorem{corollary}[theorem]{Corollary}
\theoremstyle{definition}

\newtheorem{assumption}[theorem]{Assumption}

\theoremstyle{remark}

\newcommand{\spanv}[1]{\textrm{sp}(#1)}

\DeclareMathOperator{\E}{\mathbb{E}}

\DeclareMathOperator*{\argmax}{arg\,max}

\title{On Value Iteration Convergence in Connected MDPs}

\author{%
  Arsenii Mustafin  \\
  Department of Computer Science\\
  Boston University\\
  Boston, MA 02215, USA \\
  \texttt{aam@bu.edu} \\
   \And
   Alex Olshevsky \\
   Department of Electrical and Computer Engineering \\
   Boston University, Boston, MA 02215, USA \\
   \texttt{alexols@bu.edu} \\
   \AND
   Ioannis Ch. Paschalidis \\
   Department of Electrical and Computer Engineering \\
   Boston University, Boston, MA 02215, USA \\
   \texttt{yannisp@bu.edu} \\
}

\begin{document}

\maketitle

\begin{abstract}

This paper establishes that an MDP with a unique optimal policy and ergodic associated transition matrix ensures the convergence of various versions of the Value Iteration algorithm at a geometric rate that exceeds the discount factor $\gamma$ for both  discounted and average-reward criteria.

\end{abstract}


\section{Introduction} \label{sec:introduction}

Markov decision process is a common mathematical tool to model multi-step iteration with an action-reward system, which draws increased attention due to recent popularity of Reinforcement learning. Value iteration and policy iteration are two basic algorithms which are used to solve MDP, or to find optimal optimal policy. Two algorithms are different in their nature: policy iteration finds optimal policy exact optimal policy, but requires more operations every iteration of the algorithm. Value iteration find approximate optimal policy up to desired accuracy $\epsilon$, every iteration of this algorithm requires less operations, but overall number of iterations depends on $\epsilon$. In general, value iteration converges geometrically at a rate determined by the discount factor $\gamma$. An example where this convergence rate is exact can be provided. Faster convergence can occur under additional assumptions, as shown in \cite{puterman2014markov}, Section 6.6.6, or in \cite{feinberg2020complexity}.

\subsection{Motivation and contribution} \label{ssec:motivation_and_contribution}

In this paper, we aim to analyze the additional convergence of the Value Iteration algorithm beyond the theoretically guaranteed convergence with discount factor $\gamma$. An example of this convergence is shown in Figure \ref{fig:motiv_example}.

We show that under the fairly general assumptions of an irreducible and aperiodic Markov chain implied by a unique optimal policy (Assumptions \ref{ass:irred_and_aper_MDP} and \ref{ass:uniqueness}), value iteration exhibits geometric convergence at a rate faster than $\gamma$. This additional convergence factor reduces the computational complexity required to achieve an accuracy of $\epsilon$ from:

\begin{equation*}
    \mathcal{O}\left( \frac{\log{1/\epsilon} + \log (1/(1-\gamma)) }{\log{(1/\gamma)}} \right) \quad\text{to}\quad
    \mathcal{O}\left(\frac{\log (1/\epsilon) + \log(1/(1-\gamma)) }{\log (1/\gamma) + \log(1/\tau)/N} \right)
\end{equation*}
in the discounted reward criteria case ($\gamma < 1$).

Additionally, we establish geometric convergence in the average total reward case with a total complexity of

\begin{equation*}
    \mathcal{O} \left( \frac{\log{(1/\epsilon)} }{\log{(1/\tau)}/N} \right).
\end{equation*}

\begin{figure}[b]
  \caption{Convergence of the value iteration algorithm on a random MDP with different discount rates. It can be seen that as $\gamma$ approaches one, the convergence rate remains geometric with a rate less than $\gamma$.}\label{fig:motiv_example}
  \centering
    \includegraphics[width=1.0\textwidth]{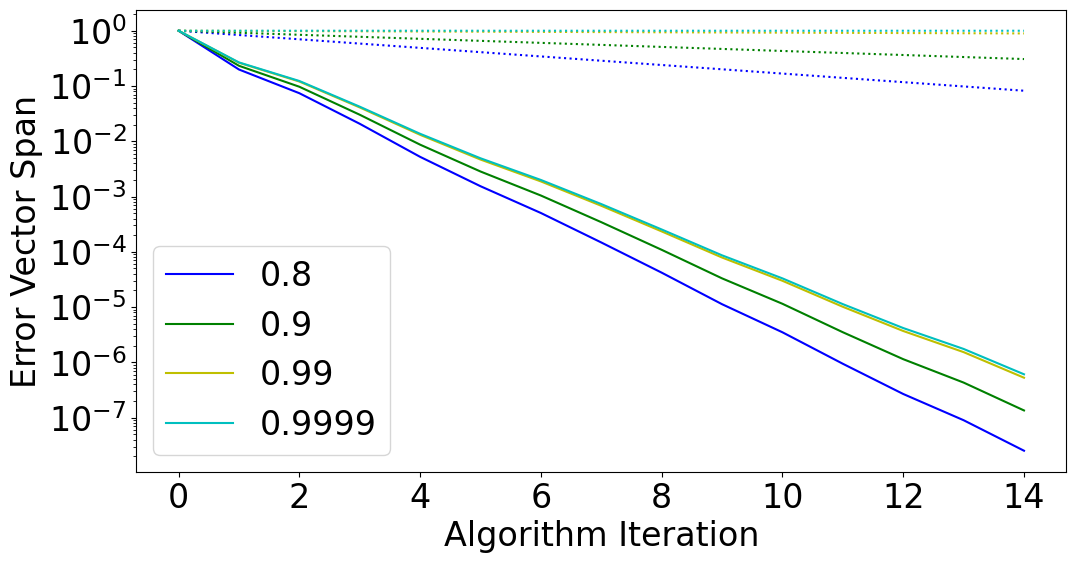}
\end{figure}

\section{Problem formulation} \label{sec:problem_formulation}

\begin{algorithm}[t]
   \caption{Value Iteration algorithm}
   \label{alg:value_iter}
\begin{algorithmic}
   \STATE {\bfseries Parameters} Learning rate $\alpha$ and desired accuracy $\epsilon$.
   \STATE {\bfseries Initialize} $V_0$, set $t=0$ and set stopping criteria $H$.
   \STATE Select subset of states $\mathcal{S}_t$ and perform  $U = \max_a r(s,a) + \gamma P_a(s) V$,
   \STATE Compute $V_{t+1} (s) = (1-\alpha) V_{t} + \alpha ( \max_a r(s,a) + \gamma P_a(s)V_{t}$) if $s \in \mathcal{S}_t$, $V_{t+1} (s) = V_t(s)$ otherwise.
   \IF{ $\spanv{V_{t+1} - V_{t}} \le H$}
    \STATE Increment $t$ by 1 and return to the previous step.
   \ELSE
    \STATE Output policy $\pi$, such that $\pi(s) = \argmax_{a}  r(s,a) + \gamma P_a(s)V_{t}$.
    \ENDIF
\end{algorithmic}
\end{algorithm}

We consider a discounted reward Markov Decision Process (MDP) defined by the tuple $(\mathcal{S}, \mathcal{A}, \mathcal{P}, r, \gamma)$, where $\mathcal{S}$ is the finite set of states, $\mathcal{A}$ is the finite set of actions, $\mathcal{P} = {P(s'|s,a)}_{s,s' \in \mathcal{S}, a \in \mathcal{A}}$ represents the transition probabilities, $r$ is the reward function, and $\gamma \in [0,1)$ is the discount rate. We denote the reward received during the transition from $s$ to $s'$ as $r(s,s')$ and the expected reward received after taking action $a$ in state $s$ as $r(s,a)$. A deterministic policy $\pi : \mathcal{S} \rightarrow \mathcal{A}$ maps states to actions. Given a policy, we can define a state value function, which is equal to

\[ V^\pi(s) := \E\left[ \sum_{t=0}^{\infty} \gamma^t r(s_t,\pi(s_t))   | s_0 = s \right] \]

in discounted infinite horizon reward criteria and 

\[ V^\pi(s) := \lim_{T \rightarrow \infty} \frac{1}{T}\E\left[ \sum_{t=0}^{T}  r(s_t,\pi(s_t))   | s_0 = s \right] \]
in average total reward criteria.

A policy $\pi^*$ is called optimal if it maximizes the values of all states. A policy $\pi^\epsilon$ is called $\epsilon$-optimal if the state values implied by $\pi$ differ from the optimal state values by less than a positive real number $\epsilon$:
\[ V^*(s) - V^{\pi}(s) \le \epsilon, \forall s \in \mathcal{S}.\]

The Value Iteration Algorithm \ref{alg:value_iter} is one of the algorithms used to find an $\epsilon$-optimal policy. The pseudocode presents a general version of the algorithm, and in the paper, we analyze three versions of it:

\begin{itemize}
\item \textbf{Synchronous without learning rate} (classical value iteration algorithm): $\alpha = 1$ and every action is updated at every step, $\mathcal{S}{t} = \mathcal{S} ; \forall t$.
\item \textbf{Synchronous with learning rate}: $\alpha \in (0,1)$ and every action is updated at every step, $\mathcal{S}{t} = \mathcal{S} ; \forall t$.
\item \textbf{Asynchronous with learning rate}: $\alpha \in (0,1)$ and the set of actions updated at each step is different.
\end{itemize}

The choice of stopping criteria $H$ depends on the reward criteria. We aim to show that under the assumption that the Markov Reward process implied by the optimal policy is irreducible and aperiodic, all algorithms exhibit geometric convergence with a rate smaller than the discount factor $\gamma$.

In our analysis, we show convergence in terms of the span of an error vector $\text{span}(e)$, which is a vector of differences between current values $V$ and optimal values $V^*$:

\[ \spanv{e} = \max{(e)} - \min{(e)}, \quad e=V-V^*. \]

\section{Main results} \label{sec:main_results}

\subsection{Classical algorithm} \label{ssec:classical_alg}

\begin{assumption} \label{ass:irred_and_aper_MDP}
The MRP implied by optimal policy $\pi^*$ is irreducible and aperiodic.
\end{assumption}

\begin{assumption} \label{ass:uniqueness}
    \textbf{Uniqueness of optimal policy:} For every state $s$ there exist only one optimal action $\pi^*(s)$. Any other action $a'$ is worse than $a^*$ by at least $\delta$:
    \begin{equation*}
         r(s,\pi^*(s)) +\gamma \sum_{s'} P(s'|s,\pi^*(s)) V^*(s') - r(s,a') - \gamma \sum_{s'} P(s'|s,a') V^*(s') \ge \delta,  \forall s, a' \ne \pi^*(s)
    \end{equation*}
\end{assumption}

In the theorem presented, we demonstrate that under the previously stated assumptions, a limited number of iterations of the value iteration algorithm is sufficient to ensure enhanced convergence by a factor of $\tau \in (0,1)$. The exact number of iterations required, denoted as $N$, is determined by the need to ensure mixing by the stochastic matrix $P^*$. This number can be upper-bounded by $n^2 - 2n + 2$, as outlined in the subsequent lemma:

\begin{lemma} \label{lem:positive_matrix}
For every irreducible and aperiodic stochastic matrix $A$, all elements of a matrix $A^{n^2-2n+2}$ are strictly positive.
\end{lemma}
\begin{proof}
This fact is stated in \cite{holladay1958powers}.
\end{proof}

\begin{theorem} \label{thm:sync_no_lr}
Convergence of the \textbf{Synchronous algorithm without a learning rate:} If Assumptions \ref{ass:irred_and_aper_MDP} and \ref{ass:uniqueness} hold, span of the error vector obtained after $N$ steps of synchronous Value Iteration algorithm has the following property:

$$ \spanv{e_N} \le \gamma^N \tau \spanv{e_0}, $$

where $\tau \in (0,1)$.
\end{theorem}

\begin{proof}
The proof is given in Subsection \ref{ssec:proof_of_thm_sync_no_lr}.
\end{proof}

The convergence of the span of the error vector, as implied by Theorem \ref{thm:sync_no_lr}, is sufficient to achieve an $\epsilon$-optimal policy with the rate faster than $\gamma$. This can be demonstrated with the assistance of the following lemma.

\begin{lemma} \label{lem:spans}
For the consecutive value vectors $V_t$ and $V_{t+1}$ the following inequality holds:
   \begin{equation} \label{eq:span_relation}
        \spanv{V_t - V_{t+1}} \le \spanv{e_t} (1 + \gamma) 
   \end{equation}
\end{lemma}

\begin{proof}
    We denote two states where vector $e_t$ reaches its maximum and minimum as $\overline{s}$ and $\underline{s}$. In addition, we denote two states corresponding to the maximum and minimum entries of vector $e_{t+1}$ as $\overline{s}'$ and $\underline{s}'$. Then, 
\begin{align*}
\spanv{V_t - V_{t+1}} &= \spanv{e_t - e_{t+1}} = \max_s (e_t(s) - e_{t+1}(s) ) - \min_s (e_t(s) - e_{t+1}(s) ) \le \\
& \le  (e_t(\overline{s}) - e_{t+1}(\underline{s}')) -
    (e_{t}(\underline{s}) - e_{t+1} (\overline{s}') ) = \spanv{e_t} + \spanv{e_{t+1}} \le \\
    & \le (1+\gamma) \spanv{e_t}  
\end{align*}
\end{proof}

Having Lemma \ref{lem:spans} we can derive convergence time of this version of the algorithm for both discounted reward criteria and average reward criteria.

\begin{corollary} \label{cor:disc_reward_complexity}
    In discounted reward criteria case set the algorithm stopping criteria $H$ equal to $\frac{\epsilon (1-\gamma)}{\gamma}$, synchronous Value iteration algorithm without learning rate will output $\epsilon$-optimal policy after at most: 
\begin{equation*}
    \frac{\log{1/\epsilon} + \log{1-\gamma} - \log 4 - \log{\spanv{e_0}}}{\log \gamma + \log{(1/\tau)}/N }
\end{equation*}
iterations, which results in total iteration complexity of

\begin{equation} \label{eq:num_iter}
    \mathcal{O} \left( \frac{\log{(1/\epsilon)} + \log{(1-\gamma)}}{\log{(1/\gamma)} + \log{(1/\tau)}/N} \right).
\end{equation}

\end{corollary}
\begin{proof}
From Lemma \ref{lem:spans} it follows that $\spanv{e_t} < \frac{\epsilon (1-\gamma)}{\gamma (1+\gamma)}$ guarantees that algorithm stopping criteria is satisfied. Proposition 6.6.5 from \cite{puterman2014markov} guarantees that policy produced by this values is $\epsilon$-optimal.
From Theorem \ref{thm:sync_no_lr} we know that after $t$ iterations span of an error vector $e_t$ might be upper bounded by:
\begin{equation*}
    \spanv{e_t} \le \gamma^t \tau^{\lfloor t/N \rfloor} \spanv{e_0}.
\end{equation*}

Therefore, after number of iterations specified in Equation \ref{eq:num_iter} $\spanv{e_t}$ is small enough, so that $\spanv{e_{t+1} - e_t}$ satisfies $\epsilon$-optimal policy criteria.
    
\end{proof}

\begin{corollary}
    In average reward criteria case set the algorithm stopping criteria $H$ equal to $\epsilon$, synchronous Value iteration algorithm without learning rate will output $\epsilon$-optimal policy after at most: 
\begin{equation*}
    \frac{\log{1/\epsilon} + \log 2 - \log{\spanv{e_0}}}{ \log{(1/\tau)}/N }
\end{equation*}
iterations, which results in total iteration complexity of

\begin{equation} \label{eq:num_iter2}
    \mathcal{O} \left( \frac{\log{(1/\epsilon)} }{\log{(1/\tau)}/N} \right).
\end{equation}

\end{corollary}

\begin{proof}
Proof is similar to the proof of Corollary \ref{cor:disc_reward_complexity}, except it uses Theorem 8.5.7 from \cite{puterman2014markov}, which states that once criteria $\spanv{e_{t+1} - e_t} < \epsilon$ is met the policy produced by $V_{t+1}$ is $\epsilon$-optimal.
\end{proof}

Note, that both stopping criteria $\frac{\epsilon (1-\gamma)}{\gamma}$ in discounted reward MDP and $\epsilon$ in average total reward MDP do not depend on values of $N$ and $\tau$, which are unknown during the run of the algorithm. Therefore, we take an advantage of the extra convergence factor $\log{(\tau)}/N$ even not knowing its exact value.

\subsection{Synchronous with learning rate} \label{ssec:sync_w_lr}

As we can see from Subsection \ref{ssec:classical_alg}, there are two main sources of error vector span convergence: it is shrinking by discount factor $\gamma$ and reverts to the mean due to mixing properties of the stochastic matrix $P_t'$. In this section we show that by adding learning rate we introduce a trade off between these two sources of convergence:  by increasing learning rate part of part of $\gamma$ convergence is being "sacrificed" to provide faster information exchange between states and bolster the mean reversion.

One immediate result of introducing the learning rate is that now it is guaranteed, that under MRP produced by optimal policy a number of updates $N_\alpha$ required to guarantee that every state affects every other state is at most $n-1$.

\begin{theorem} \label{thm:sync_w_lr}
Convergence of the \textbf{Synchronous algorithm with a learning rate:} If Assumptions \ref{ass:irred_and_aper_MDP} and \ref{ass:uniqueness} hold, span of the error vector obtained after $n$ steps of synchronous Value Iteration algorithm with learning rate $\alpha \in (0,1)$ has the following property:

\begin{equation*}
    \spanv{e_{N_\alpha}} \le \gamma^{N_\alpha} \tau_\alpha \spanv{e_0},
\end{equation*}

where $\gamma^{N_\alpha}\tau_\alpha \in (0,1)$.

\end{theorem}

\begin{proof}
    Proof of this theorem is similar to the proof of Theorem \ref{thm:sync_no_lr}. Full version of the proof is given in Appendix \ref{ssec:sync_w_lr_proof}.
\end{proof}

Having this theorem we can state two convergence Corollaries analogous to Colloraries in previous section wit identical proofs.

\begin{corollary} \label{cor:disc_reward_w_l_complexity}
    In discounted reward criteria case set the algorithm stopping criteria $H$ equal to $\frac{\epsilon (1-\gamma)}{\gamma}$. Then synchronous Value iteration algorithm with learning rate will output $\epsilon$-optimal policy after at most: 
\begin{equation*}
    \frac{\log{1/\epsilon} + \log{1-\gamma} - \log 4 - \log{\spanv{e_0}}}{\log \gamma + \log{(1/\tau_\alpha)}/N_\alpha }
\end{equation*}
iterations, which results in total iteration complexity of

\begin{equation} \label{eq:num_iter_wl}
    \mathcal{O} \left( \frac{\log{(1/\epsilon)} + \log{(1-\gamma)}}{\log{(1/\gamma)} + \log{(1/\tau_\alpha)}/N_\alpha} \right).
\end{equation}

\end{corollary}

\begin{corollary}
    In average reward criteria case set the algorithm stopping criteria $H$ equal to $\epsilon$. Then synchronous Value iteration algorithm with learning rate will output $\epsilon$-optimal policy after at most: 
\begin{equation*}
    \frac{\log{1/\epsilon} + \log 2 - \log{\spanv{e_0}}}{ \log{(1/\tau_\alpha)}/N_\alpha }
\end{equation*}
iterations, which results in total iteration complexity of

\begin{equation} \label{eq:num_iter_wl_2}
    \mathcal{O} \left( \frac{\log{(1/\epsilon)} }{\log{(1/\tau_\alpha)}/N_\alpha} \right).
\end{equation}

\end{corollary}

\subsection{Asynchronous with learning rate} \label{ssec:async}

The convergence analysis of asynchronous algorithm is harder than previous cases due to its dependence on choices of states $\mathcal{S}_t$ every iteration. For the algorithm to converge we need to guarantee meaningful choice strategy. Otherwise, for example when only a few states being updated infinite number of times,  while others are not being updated at all, convergence cannot be achieved. To guarantee proper state selection we make a following assumption:

\begin{assumption} \label{ass:update_freq}
   There exists a natural number $B$ such that during every $B$ updates every state is being updated at least $n$ times. 
\end{assumption}

Also, if in the synchronous update case contraction factor $\gamma$ being applied to all states simultaneously and, thus, decreases a span of a error vector, in asynchronous case it might be increasing, if algorithm prioritize choice of some states compare to others. Thus, two factors, $\gamma$-contraction and mean reversion via information exchange, which previously were both contributing to faster convergence, now might work in opposite direction: due to unbalanced application of the updates, $\gamma$-contraction might increase the span of the error vector. In this case dynamics of error vector span might be described by the following theorem:

\begin{theorem} \label{thm:async_w_lr}
Convergence of the \textbf{Asynchronous algorithm with a learning rate:} If Assumptions \ref{ass:irred_and_aper_MDP}, \ref{ass:uniqueness} and \ref{ass:update_freq} hold, span of the error vector obtained after $B$ steps of asynchronous Value Iteration algorithm with learning rate $\alpha \in (0,1)$ has the following property:

\begin{equation*}
    \spanv{e_B} \le \gamma^N\tau'\spanv{e_0} + (1-\gamma^B)\min (e_0),
\end{equation*}

where $\gamma^N\tau' \in (0,1)$.

\end{theorem}

This Theorem immediately gives a convergence result in average reward case stated in the following corollary.

\begin{corollary}
    In average reward criteria case set the algorithm stopping criteria $H$ equal to $\epsilon$. Then asynchronous Value iteration algorithm with learning rate will output $\epsilon$-optimal policy after at most: 
\begin{equation*}
    \frac{\log{1/\epsilon} + \log 2 - \log{\spanv{e_0}}}{ \log{(1/\tau_\alpha')}/B_\alpha }
\end{equation*}
iterations, which results in total iteration complexity of

\begin{equation} \label{eq:num_iter_wl_async}
    \mathcal{O} \left( \frac{\log{(1/\epsilon)} }{\log{(1/\tau_\alpha')}/B_\alpha} \right).
\end{equation}

\end{corollary}

At the same time, the term $(1-\gamma^B)\min (e_0)$ does not allow to establish the convergence guarantees in discounted case using the same proof strategy.

\section{Proofs}  \label{sec:proofs}

\subsection{Key idea of the proofs} \label{ssec:key_idea}

Let's analyze a single update on a single state value.

\begin{align} \label{eq:upper_bound}
v_{t+1}(s) &= v^*(s) + e_{t+1}(s) =  r(s,\pi_t(s)) + \gamma\sum_{s'} P_t(s,s') v_t(s') \nonumber \\
&= r(s,\pi_t(s)) +\gamma \sum_{s'} P_t(s,s') (v^*(s') + e_{t}(s')) \nonumber \\
&= r(s,\pi_t(s)) + \gamma \sum_{s'} P^*(s,s') v^*(s') + r(s,\pi^*(s)) - r(s,\pi^*(s))  \nonumber \\
&+ \gamma \left( \sum_{s'} P_t(s,s') v^*(s') - \sum_{s'} P^*(s,s') v^*(s') \right) 
+\gamma  \sum_{s'} P_t(s,s')e_{t}(s') \nonumber \\
&= v^*(s) +\gamma  \sum_{s'} P_t(s,s')e_{t}(s') + \nonumber \\
&  \underbrace{r(s,\pi_t(s)) - r(s,\pi^*(s)) + \gamma \left( \sum_{s'} P_t(s,s') v^*(s') -\sum_{s'} P^*(s,s') v^*(s') \right) }_{\Delta(a)}  \nonumber \\
&\implies e_{t+1}(s) \le \gamma\sum_{s'} P_t(s,s')e_{t}(s')
\end{align}
with the equality achieved when action $a=\pi_t(s)=\pi^*(s)$. Note, that the term $\Delta(a)$ corresponds to the loss caused by one time choice of action $a$ instead of optimal action in state $s$, in the recent literature it is called an \textbf{advantage} of policy $\pi_t$ to $\pi^*$. The inequality carries the ideas of convergence of error vector span. Firstly, because non-increasing and non-decreasing properties of stochastic matrix, span will be contracting by $\gamma$ each iteration (if $\gamma < 1$). Secondly, the convergence will follow from mixing properties of matrix $P^*$ if the optimal action is chosen and from additional term $\Delta(a) \ge \delta$ otherwise.

\subsection{Proof of Theorem \ref{thm:sync_no_lr}} \label{ssec:proof_of_thm_sync_no_lr}

First, we derive a lower bound for a new error vector:

\begin{align} \label{eq:lower_bound}
v_{t+1}(s) &= v^*(s) + e_{t+1}(s) =  r(s,\pi_t(a)) + \gamma\sum_{s'} P_t(s,s') v_t(s') \nonumber\\
&\ge r(s,\pi^*(s)) +\gamma \sum_{s'} P^*(s,s')v_t(s') \nonumber \\
&= r(s,\pi^*(s)) +\gamma \sum_{s'} P^*(s,s') (v^*(s') + e_{t}(s')) \nonumber \\
&=r(s,\pi^*(s)) +\gamma \sum_{s'} P^*(s,s') v^*(s') +\gamma \sum_{s'} P^*(s,s')e_{t}(s') \nonumber \\
&\implies e_{t+1}(s) \ge \gamma\sum_{s'} P^*(s,s')e_{t}(s')
\end{align}

Combining Inequalities \ref{eq:lower_bound} and \ref{eq:upper_bound} and them rewriting in matrix form we yield:

\begin{align} \label{eq:error_dynamic_bounds}
\gamma P^* e_t\le e_{t+1} \le \gamma P_t e_t
\end{align} 

Since all elements of $e_{t+1}$ lie between two quantities, it might be written as a convex combination of those quantities, or:

\begin{align} \label{eq:error_transition}
    e_{t+1} = \gamma [D_t P^* + (I-D_t) P_t] e_t = \gamma P_t' e_t,
\end{align}

where $D_t$ is a diagonal matrix. We now analyze diagonal entries of $D_t$. As it was shown in Subsection \ref{ssec:key_idea}, in the case when for state $s_i$ vector $v_i$ implies an optimal actions (thus, $P^*e_t = P_t e_t$), we choose $D_t(i,i)=1$. In the other case, recall that the second inequality is strict only when action choice is optimal, otherwise value of $e_{t+1}(s_i)$ will be smaller by at least $\delta$:

$$e_{t+1}(s_i) \le \gamma\sum_{s'} P^*(s,s')e_{t}(s') - \gamma \delta ,$$

which, in turn, implies that the value of $D_t(i,i)$ cannot be zero. We can derive the lower bound on it as (for simplicity we denote $\sum_{s'} P(s,s')e_{t}(s')$ as $P(s)e_t$ for any transition matrix $P$):

\begin{align*}
&\begin{rcases}
 e_{t+1} (s) = \gamma (D_t(i,i)P^*(s) e_t + (1-D_t(i,i))P_t(s) e_t) \\
 e_{t+1}(s_i) \le \gamma P_t(s) e_t - \gamma \delta
\end{rcases} \implies \\
&D_t(i,i) \gamma (P_t(s) e_t - P^*(s) e_t) \ge \gamma \delta \implies \\
& \forall t, i:D_t(i,i) \ge \frac{\delta}{P_t(s) e_t - P^*(s) e_t} \ge \frac{\delta}{\spanv{e_t}} \ge \frac{\delta}{\spanv{e_0}}
\end{align*}  

Having the bound on $D_t(i,i)$ we can establish that for any two states $s,s'$ if the transition under optimal policy has non-zero probability $P^*(s,s')$, this transition will also have non-zero probability in matrix $P'$:

\begin{align} \label{eq:trans_matrix_positivity}
    P_t'(s,s') \ge \frac{\delta(1-\gamma)}{r_{\rm max} - r_{\rm min}} \min_{s,s'} (P^*(s,s')) = \delta'.
\end{align} 

Thus, the result of $N$ steps of value iteration algorithm in terms of error vector might be expressed as:

\[ e_N = \gamma^N P_N'P_{N-1}'\dots P_1'e_0 .\]

Every entry of $e_N(s)$ might be expressed as convex combination of entries of $e_0$ (we denote coefficients $\lambda$), where each coefficient $\lambda_{s'}$ represents a sum of probabilities of patches $u$ from state $s'$ to $s$ of length $N$, denoted as $U_N(s',s)$. Since the matrix $P^{*N}$ is positive, for any two states $s_i, s_j$ there exists at least one path $u_{ij}: (s_i,s_{u_1}, \dots, s_{u_{N-1}}, s_j)$ of length $N$, such that of any of two consecutive entries $s_{p_i}$ and $s_{p_{i+1}}$ the transition from first to second will have non-zero probability, \textit{i.e.} $P^*(s_{p_i}, s_{p_{i+1}})>0$. Therefore, 

\begin{align*}
    e_N(s) = &\gamma^N \sum_{s'\in \mathcal{S}} \lambda_{s'} e_0(s') = \gamma^N \sum_{s'\in \mathcal{S}} \left(\sum_{u \in U_N(s',s)} P(u) \right) e_0(s')=\\
    &\gamma^N \sum_{s'\in \mathcal{S}} \left(\sum_{u \in U_N(s',s)} \prod_{i=0}^{N-1} P_{i+1}'(s_{u_i}, s_{u_{i+1}}) \right) e_0(s')
\end{align*}

Earlier we established two facts: that for every state $s'$ a paths of length $N$ from $s'$ to $s$, \textit{i.e.} $\prod_{i=1}^{N} P^*(s_{u_i}, s_{u_{i+1}}) > 0$. The second fact is that for any two states $s,s'$ and time step $t$, $P^*(s,s') > 0 \implies P_t'(s,s')\ge \delta'$. Combining these facts together we have that every coefficient $\lambda_{s'} \ge \delta'^N$. Thus we can rewrite the error representation as:

\begin{align*}
    e_N(s) &= \gamma^N\sum_{s'\in \mathcal{S}}\lambda_{s'} e_0(s') = \gamma^N\sum_{s'\in \mathcal{S}} \delta'^N e_0 (s') + (\lambda_{s'}-\delta'^N) e_0(s') = \\ 
    &= \gamma^Nn \delta'^N \overline{e}_0 + 
\gamma \sum_{s'\in \mathcal{S}} (\lambda_{s'}-\delta'^N) e_0(s'),\\
\end{align*} 

where $\overline{e}_0$ denotes the mean of vector $e_0$. This implies

\begin{align*}
&\gamma^N (n\delta'^N\overline{e}_0 + (1-n\delta'^N)\min (e_0) )  \le e_N(s) \le \gamma^N (n\delta'^N\overline{e}_0 + (1-n\delta'^N)\max (e_0) ) \implies \\
& \spanv{e_N} \le  \gamma^N(1-n\delta'^N)\max (e_0) - \gamma^N(1-n\delta'^N)\min (e_0) = 
\gamma^N(1-n\delta'^N)\spanv{e_0}
\end{align*}

Defining $\tau =(1-n\delta'^N) $ concludes the proof. \textbf{Q.E.D.}

\section{Conclusion} \label{sec:conclusion}

In this paper, we analyzed the convergence of the Value Iteration algorithm. We show that under the assumption of sufficient mixing provided by the optimal policy, the Value Iteration algorithm exhibits geometric convergence at a rate faster than the discount factor $\gamma$.

\printbibliography

\newpage

\appendix

\section{Theorem proofs}

\subsection{Proof of Theorem \ref{thm:sync_w_lr}} \label{ssec:sync_w_lr_proof}

With learning rate introduced one iteration of the algorithm is:
\begin{equation*}
    V_{t+1}(s) \leftarrow V_{t}(s)(1-\alpha) + \alpha \max_a \left[ r(s,a) +\gamma \sum_{s'} P(s'|s,a) V_{t}(s') \right]
\end{equation*}

and error transition dynamics equality similar to \ref{eq:error_transition} becomes:

\begin{align} \label{eq:error_transition_lr}
    e_{t+1} = \big((1-\alpha)I + \gamma \alpha [D_t P^* + (I-D_t) P_t]\big) e_t = \gamma P_{t,\alpha}' e_t,
\end{align}

Note that $P_{t,\alpha}'$ is a stochastic matrix only in a case when $\gamma = 1$, but it is sufficiently close to it since we assume that $\gamma$ is almost $1$. Note, that now we actually need only $N_\alpha=n-1$ updates to guarantee that the matrix $P_{t,\alpha}'^N$ is positive, since elements on the main diagonal coming from $(1-\alpha)/\gamma I$ influence error dynamics the similar way as having a loop in every state, thus every state will be affected by every other state in at most $n-1$. 
Consequently, the minimum values of $\delta'$ needs to be updated, now we have that $P_{t,\alpha}'(s,s')\ge\alpha \delta'$ for states $s,s':s\ne s'$ and $P_{t,\alpha}'(s,s) \ge (1-\alpha)\gamma $. Let's define $\delta'_\alpha$ as a minimum of these two quantities. Thus, an expression of the of an error associated with state $s$ after $N$ iterations becomes:

\begin{align*}
    e_N(s) &= \gamma^N\sum_{s'\in \mathcal{S}}\lambda_{s'} e_0(s') = \gamma^N\sum_{s'\in \mathcal{S}}\delta^{'N}_\alpha e_0 (s') + (\lambda_{s'}-\delta^{'N}_\alpha) e_0(s') = \\ 
    &= \gamma^Nn \delta^{'N}_\alpha \overline{e}_0 + 
\gamma^N\sum_{s'\in \mathcal{S}} (\lambda_{s'}-\delta^{'N}_\alpha) e_0(s'),\\
\end{align*} 

Note, that now the sum of coefficients $\lambda_{s'}$ is not $1$, but $[(1-\alpha)/\gamma + \alpha]^N$. This gives us a final convergence rate of:

\begin{equation*}
    \spanv{e_N} \le \gamma^N ([(1-\alpha)/\gamma + \alpha]^N-n\delta_\alpha^{'N}) \spanv{e_0}
\end{equation*}

Defining $([(1-\alpha)/\gamma + \alpha]^N-n\delta_\alpha^{'N})$ as $\tau_\alpha$ we have the claimed result.

\end{document}